%% file: root.tex
\documentclass[letterpaper, 10 pt, conference]{ieeeconf}  

\IEEEoverridecommandlockouts                              

\usepackage{multicol}
\usepackage[bookmarks=true]{hyperref}
\usepackage{amsmath}
\usepackage{textcomp}
\usepackage{graphicx}
\usepackage[font=footnotesize]{subcaption}
\usepackage[font=footnotesize]{caption}
\usepackage{hyperref}
\usepackage{amssymb}
\usepackage{booktabs}
\usepackage[normalem]{ulem}
\usepackage{verbatim}
\usepackage[export]{adjustbox}
\usepackage{amsmath}
\usepackage{url}
\usepackage{siunitx}
\usepackage[utf8]{inputenc}
\usepackage[TS1,T1]{fontenc}
\usepackage{array, booktabs}
\usepackage{caption}
\usepackage{float}
\usepackage{bm}
\usepackage{multirow}
\usepackage{url}
\usepackage{amsmath,amssymb,amsfonts,amsfonts}
\usepackage{algorithm}
\usepackage[noend]{algorithmic}
\usepackage{xspace}
\usepackage{wrapfig}
\usepackage{color}
\usepackage{hyperref}
\let\labelindent\relax
\usepackage{enumitem}
\usepackage{times}
\usepackage{helvet}
\usepackage{courier}

\newcommand{\eref}[1]{Eq.~(\ref{#1})} 
\newcommand{\sref}[1]{Sec.~\ref{#1}} 
\newcommand{\figref}[1]{Fig.~\ref{#1}} 

\usepackage{ifthen}
\usepackage[usenames,dvipsnames]{xcolor}
\newboolean{include-notes}
\setboolean{include-notes}{true} 
\newcommand{\ssnote}[1]{\ifthenelse{\boolean{include-notes}}%
 {\textcolor{blue}{\textbf{SS: #1}}}{}}
\newcommand{\rhnote}[1]{\ifthenelse{\boolean{include-notes}}%
 {\textcolor{red}{\textbf{RH: #1}}}{}}
\newcommand{\rhadd}[1]{\ifthenelse{\boolean{include-notes}}%
 {\textcolor{orange}{\textbf{ADD: #1}}}{}}
\newcommand{\osnote}[1]{\ifthenelse{\boolean{include-notes}}%
 {\textcolor{purple}{\textbf{OS: #1}}}{}}

\newcommand{\frechet}{Fr\'{e}chet }
\newcommand{\frechetd}{F_{d}\xspace}
\newcommand{\bigoh}[1]{\mathcal{O}(#1)}
\newcommand{\cspaceWord}{\Cspace-space\xspace}

\title{\LARGE \bf
Minimizing Task-Space \frechet Error via Efficient Incremental Graph Search}

\author{Rachel Holladay$^{1}$, Oren Salzman$^{2}$ and Siddhartha Srinivasa$^{3}$
\thanks{$^{1}$Rachel Holladay is with the Computer Science and Artificial Intelligence Laboratory, Massachusetts Institute of Technology, {\tt\small rhollada@mit.edu}. $^{2}$Oren Salzman is with the Robotics Institute, Carnegie Mellon University. {\tt\small osalzman@andrew.cmu.edu}. $^{3}$Siddhartha Srinivasa is with the Paul G. Allen School of Computer Science \& Engineering, University of Washington, {\tt\small siddh@cs.uw.edu}}}

\begin{document}

\maketitle
\thispagestyle{empty}
\pagestyle{empty}

\begin{abstract}
We present an anytime algorithm that generates a collision-free configuration-space path that closely follows a desired path in task space, according to the discrete \frechet distance.
By leveraging tools from computational geometry, we approximate the search space using a cross-product graph.
We use a variant of Dijkstra's graph-search algorithm to efficiently search for and iteratively improve the solution. 
We compare multiple proposed densification strategies and empirically show that our algorithm outperforms a set of state-of-the-art planners on a range of manipulation problems. 
Finally, we offer a proof sketch of the asymptotic optimality of our algorithm. 
\end{abstract}

\input{macros.tex}
\input{introduction.tex}
\input{problem_definition.tex}
\input{graph_creation.tex}
\input{search.tex}
\input{densification.tex}
\input{proof.tex}
\input{results.tex}
\input{discussion.tex}

\section*{ACKNOWLEDGMENT}
This material is based upon work supported by ONR BAA 13-0001 and undergraduate research grants from CRA-W's CREU and CMU's SRC URO programs. We thank Chris Dellin for his assistance with the LPA* implementation. We would also thank the members of the Personal Robotics Lab and the MCube Lab and Ananya Kumar for helpful discussion and advice.

\textVersion{
\addtolength{\textheight}{-12cm}
}{}

{\footnotesize
\bibliographystyle{ieeetr}
\bibliography{references}}

\textVersion{}{
\appendices
\input{appendix_proof.tex}

\input{appendix_figs.tex}
}

\end{document}

%% file: macros.tex
\newcommand{\calA}{\ensuremath{\mathcal{A}}\xspace}
\newcommand{\calC}{\ensuremath{\mathcal{C}}\xspace}
\newcommand{\calE}{\ensuremath{\mathcal{E}}\xspace}
\newcommand{\calG}{\ensuremath{\mathcal{G}}\xspace}
\newcommand{\calR}{\ensuremath{\mathcal{R}}\xspace}
\newcommand{\calM}{\ensuremath{\mathcal{M}}\xspace}
\newcommand{\calX}{\ensuremath{\mathcal{X}}\xspace}
\newcommand{\calS}{\ensuremath{\mathcal{S}}\xspace}
\newcommand{\calQ}{\ensuremath{\mathcal{Q}}\xspace}
\newcommand{\calT}{\ensuremath{\mathcal{T}}\xspace}
\newcommand{\calL}{\ensuremath{\mathcal{L}}\xspace}
\newcommand{\calB}{\ensuremath{\mathcal{B}}\xspace}
\newcommand{\ch}{\mathrm{ch}}
\newcommand\norm[1]{\left\lVert#1\right\rVert}

\newcommand{\Cfree}{\ensuremath{\calC_{\rm free}}\xspace}
\newcommand{\Cforb}{\ensuremath{\calC_{\rm forb}}\xspace}
\newcommand{\Cspace}{\ensuremath{\calC}\xspace}
\newcommand{\Tspace}{\ensuremath{\calT}\xspace}
\newcommand{\FK}{\text{FK}}
\newcommand{\IK}{\text{IK}}
\newcommand{\Rpath}{\ensuremath{\bar{\xi}}\xspace}
\def\distC{\mathop{\mathrm{d_{\Cspace}}}}
\def\distT{\mathop{\mathrm{dist_{\Tspace}}}}
\def\distF{\mathop{\mathrm{dist_{F}}}}

\newcommand{\Man}{\ensuremath{\calM_{\Rpath}}\xspace}
\newcommand{\WellMan}{\ensuremath{\tilde{\calM}_{\Rpath}}\xspace}

\newcommand{\DOF}{\textit{dof}\xspace}
\newcommand{\DOFs}{\textit{dofs}\xspace}

\newcommand{\eg}{{e.g.}\xspace}
\newcommand{\ie}{{i.e.}\xspace}
\newcommand{\etc}{{etc.}\xspace}
\newcommand{\etal}{{et~al.}\xspace}

\def\naive{{na\"{\i}ve}\xspace}

\def\oren#1{\textcolor{magenta}{#1}}
\def\ssnote#1{\textcolor{red}{#1}}

\newtheorem{thm}{Theorem}
\newtheorem{lem}{Lemma}
\newtheorem{observation}[thm]{Observation}

\def\cl{\mathop{\mathrm{cl}}}
\def\nbr{\mathop{\mathrm{nbr}}}
\def\prnt{\mathop{\mathrm{p}}}
\def\anc{\mathop{\mathrm{anc}}}
\def\dist{\mathop{\mathrm{dist}}}
\def\cmprs{\ensuremath{\chi}}

\newcommand{\ignore}[1]{}

\newboolean{WAFR}
\newboolean{ARXIV}

\setboolean{WAFR}{false}
\ifthenelse{\boolean{WAFR}}
    {\setboolean{ARXIV}{false} }
    {\setboolean{ARXIV}{true}  }

\newcommand{\textVersion}[2]
{\ifthenelse{\boolean{WAFR} }{#1}{}\ifthenelse{\boolean{ARXIV}}{#2}{}}

%% file: introduction.tex

\section{Introduction}
\label{sec:introduction}


The classical formulation of the motion-planning problem calls for planning a collision-free (possibly optimal) path between a given start and target configuration~\cite{HSS17} in a robot's configuration space (\cspaceWord).
However for robot arms, the path of the end-effector is often of greater relevance. For example, the end-effector path might be subject to constraints such as keeping a coffee mug upright, or might even be restricted to a specific path such as pulling a door open, writing on a whiteboard, or welding a seam on a car.

We focus on the latter problem: enabling a redundant robotic manipulator to follow a reference path in task space. There are two state-of-the-art approaches to solving this problem:
\begin{enumerate}
    \item \textbf{Projection-based approaches} exploit the kinematic redundancy of the manipulator~\cite{maciejewski1985obstacle,nakamura1987task,ahmad1989coordinated,burdick1989inverse} to drive the robot along the desired path~\cite{chiacchio1991closed,guo1993joint,roberts1993repeatable,seereeram1995global,siciliano1990kinematic}. Although these are typically efficient and can follow the desired path accurately, they are \emph{myopic} and can fail due to joint limits or collisions~\cite{martin1989resolution,oriolo2017repeatable,rakita2018relaxedik}.
    \item \textbf{Graph-based approaches} sample the task-space path, compute a set of inverse kinematics (IK) solutions in the \cspaceWord for each sample along the path, create a graph by connecting nearby configurations via a simple planner (like a straight line), and solve for the shortest-feasible path on this graph~\cite{oriolo2002probabilistic,descartes}. 
    Although they can solve more intricate problems via organized search, they are typically much slower, when compared to projection based approaches. More importantly, their optimization criteria (shortest path in \cspaceWord) lacks any notion of ``following'' the task-space reference path.
\end{enumerate}

\begin{figure}[!tb]
    \centering
		\begin{subfigure}[h]{0.8\columnwidth}
		      \includegraphics[width=0.99\columnwidth]{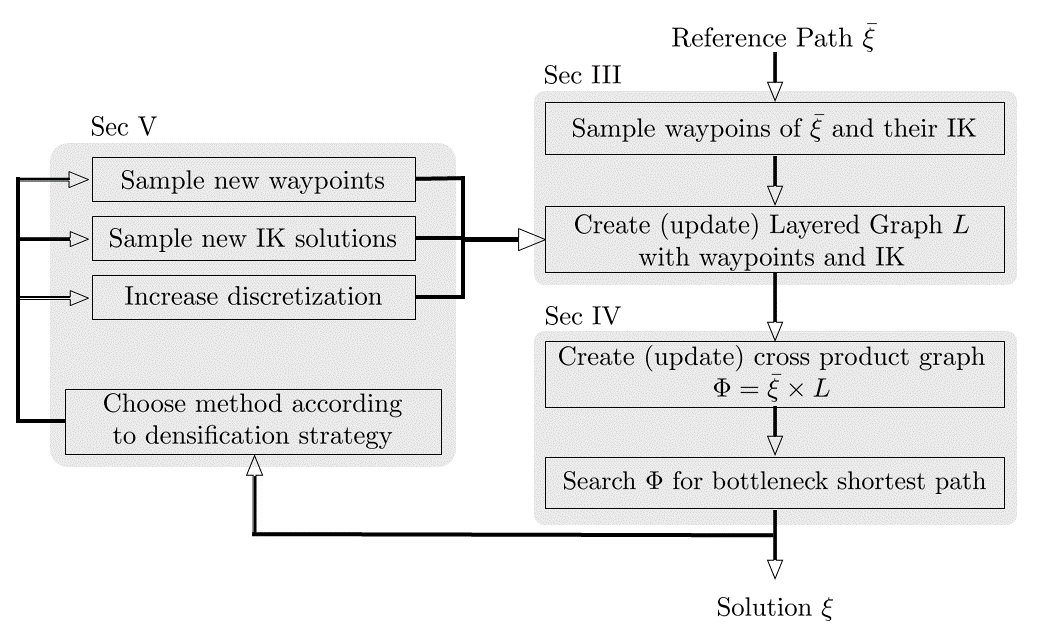}
		\end{subfigure} \\
    \begin{subfigure}[h]{0.3\columnwidth}
          \includegraphics[height=2.8cm]{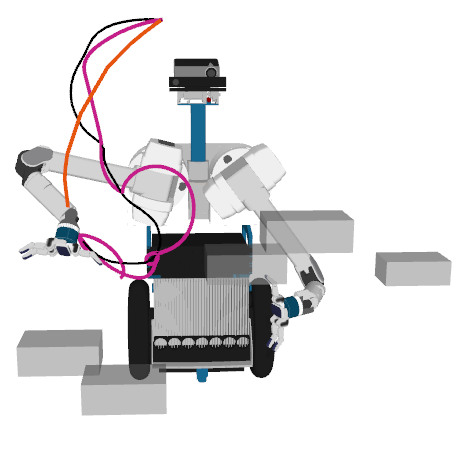}
          \caption{Initial Paths}\label{fig:initial_herb15}      
    \end{subfigure}
    \begin{subfigure}[h]{0.3\columnwidth}
        \includegraphics[height=2.8cm]{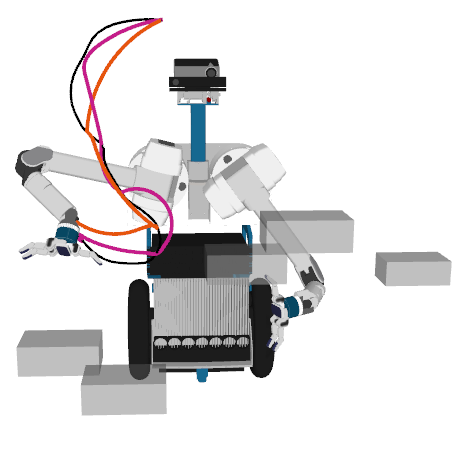}
        \caption{Midway Progress} \label{fig:mid_herb15}
    \end{subfigure}
    \begin{subfigure}[h]{0.3\columnwidth}
        \includegraphics[height=2.8cm]{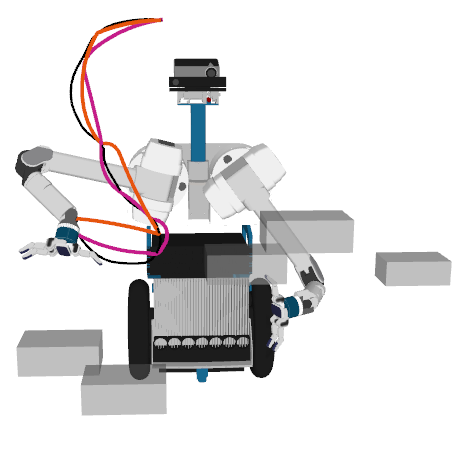}
        \subcaption{Final Paths} \label{fig:final_herb15}
    \end{subfigure}
\caption{On top is a flowchart of our algorithm. We create data structures that allow us to efficiently compute a path that minimizes the \frechet distance to the reference path and then to incrementally reduce this distance. Each grey box outlines a major step: 
(1)~generating candidate paths, 
(2)~searching over paths and 
(3)~densifying. 
On the bottom we see the progression of our planner (pink) and the planner from Holladay et al.~\cite{holladay2016distance} (orange) as they trace out the reference path (black) in the presence of obstacles (grey).}
\label{fig:fig1}
\vspace{2mm}
\end{figure}

The goal of our paper is to make graph-based approaches more efficient while still being sufficiently accurate. Central to our approach is the simple, yet fundamental question: \emph{What does it mean to approximately follow a path?} We can rephrase this more formally as
\begin{quote}
What is the right distance metric for comparing two paths in task space?
\end{quote}

Informally, let us say that we would like to stay within an $\varepsilon$-ball of any point on the reference path. The \emph{one-way Hausdorff distance} satisfies this property~\cite{hausdorff}. However, we might end up shortcutting large sections of the path. Now, if we force such proximity for both the reference and the target path, via the \emph{two-way Hausdorff distance}, we avoid shortcutting~\cite{hausdorff}. However, this does not preserve \emph{monotonicity} of traversal. If we additionally enforce monotonicity, we end up with the \emph{\frechet distance}~\cite{frechet1906quelques}.

Holladay et al.~\cite{holladay2016distance} showed the practical superiority of the \frechet distance over other metrics for trajectory optimization of manipulator motion. However, their approach, like other optimization-based approaches, suffers from local minima (\sref{sec:trajopt_bug}).

In this work we suggest to approximate the search space of candidate paths by a layered graph that organizes IK solutions by their task-space location along the reference path.
By representing both the layered graph and our reference path as simplicial complexes~\cite{har2014frechet}, we can construct  the cross product of these two complexes.
This, in turn, allows us to efficiently compute the (discrete) \frechet distance between the set of candidate paths in the layered graph and our reference path via  a simple Bottleneck Shortest Path algorithm.  

We present an anytime algorithm for incrementally densifying these structures and improving our solution  and prove that our approach is asymptotically optimal, given some assumptions. Empirically, we evaluate our approach on a seven degree-of-freedom manipulator and demonstrate its efficacy when compared to existing state-of-the-art algorithms on multiple paths and parameter settings. For a summary of our algorithmic approach, see \figref{fig:fig1} (top).

Our key insight is that marrying the correct metric (\frechet distance) with the correct search algorithm (bottleneck shortest path) enables us to focus our computation on parts of the space that are most relevant for the problem, thereby gaining better efficiency.

%% file: problem_definition.tex

\section{Problem Definition and Algorithmic Background}
\label{sec:background}
In this section we provide the basic definitions (Sec~\ref{sec:definitions}) which allow us to formally state our problem (Sec.~\ref{sec:problem_statement}).
We 
define the \frechet distance (Sec.~\ref{sec:distance_metrics}) and briefly describe the approach proposed by Holladay et al.~\cite{holladay2016distance}.
We explain a key shortcoming of their work, which motivates our approach (Sec.~\ref{sec:trajopt_bug}).

\subsection{Definitions}
\label{sec:definitions}

A configuration $q$ is a $d$-dimensional point that completely describes the location of the robot and the \cspaceWord~\Cspace is the set of all configurations~\cite{lozano1990spatial}. 
A task-space point $\tau \in SE(3)$ describes the position and orientation of the robot's end effector and the task space is the set of all such points.
Paths in \cspaceWord and task space are continuous mappings
 $\xi : [0, 1] \rightarrow$~\Cspace
and
 $\xi : [0, 1] \rightarrow SE(3)$, respectively\footnote{For simplicity, we use the same notation for paths in  \cspaceWord and in task space. The specific space will be clear from the context.}.
 
The robot induces 
a forward kinematics $\text{FK}: \Cspace \rightarrow  SE(3)$ 
and
an inverse kinematics $\text{IK}: SE(3) \rightarrow 2^\Cspace$ 
that map
a configuration to a unique task-space pose
and
a task-space pose to a set of configurations, respectively. 
By a slight abuse of notation we will use $\text{FK}(\cdot)$ to map a \cspaceWord path into a task-space path. 
Equipped with these definitions, we can define our problem.


\subsection{Problem Statement}
\label{sec:problem_statement}

We are given a robot and a reference path in task space~$\bar{\xi}$
that is a polyline given as a sequence of waypoints. 
Let $\Xi_{\bar{\xi}} \subset ~\Cspace$ be the set of all collision-free paths in \cspaceWord that have the same start and end task-space poses as $\bar{\xi}$.
Our objective is to compute 
\begin{equation}
	\xi^* = \arg \min_{\xi \in \Xi_{\bar{\xi}}} || \text{FK}(\xi), \bar{\xi}||.
\label{eqn:pdef}
\end{equation}
Namely, we seek a collision-free path $\xi \in$~\Cspace whose forward kinematics maps to a path in task space, $\text{FK}(\xi)$, that follows~$\bar{\xi}$ as close as possible, given some similarity metric $|| \cdot, \cdot ||$.
Similarly to~$\bar{\xi}$, our produced path $\xi$ is a polyline represented by a sequence of waypoints. 
To validate that paths are collision-free, we assume that we are given access to a discriminative black-box collision detector that, given a configuration~$q \in \Cspace$, returns whether or not the robot, placed in~$q$, would be in collision.
The distance metric $|| \cdot, \cdot ||$ used to compare paths is the \frechet distance, described below.
For a discussion motivating the use of the \frechet distance in this context, see Sec.~\ref{sec:introduction} and Holladay et al.~\cite{holladay2016distance}.

\subsection{Distance Metrics}
\label{sec:distance_metrics}

To describe the \frechet distance, we borrow a common analogy  where a dog is walking along a path~$\xi_{0}$ at speed parameterization $\alpha$ and its owner is walking along another path~$\xi_{1}$ at speed parameterization $\beta$. 
The two are connected via a leash. The Fr\'{e}chet distance is the shortest possible leash via some distance metric $d_{TS}$ such that there exists a parameterization $\alpha$ and $\beta$ so that the two stay connected and move monotonically. 
More formally the continuous \frechet distance between $\xi_{0}$ and $\xi_{1}$ is given by:
\begin{equation}
    F(\xi_{0},\xi_{1}) = \inf_{\alpha, \beta}\,\,\max_{t \in [0,1]} \,\,  \Bigg \{d_{TS} \Big ( \xi_{0}(\alpha(t)), \, \xi_{1}(\beta(t)) \Big ) \Bigg \}.
\label{eqn:frechet}
\end{equation}

As is common in motion planning~\cite{sucan2012open},
given two points $x, y \in SE(3) = {\mathbb{R}}^3 \times SO(3)$,
we define their distance, $d_{TS}(x, y)$, 
as the weighted sum\footnote{In our setting, we prioritize translational distance over rotational distance using a ratio of 0.17, which corresponds to 3mm mapping to 1 degree.} 
of the Euclidean metric in ${\mathbb{R}}^3$ 
and the standard great circle solid angle metric in $SO(3)$ for the respective components.

Since computing the continuous \frechet distance is notoriously difficult, especially in non-Euclidean spaces~\cite{solovey2016sampling} and our path representation is given as a series of waypoints, we
approximate $F(\cdot, \cdot)$ using the the discrete \frechet distance~$\frechetd(\cdot, \cdot)$, where the ``leash'' is only considered between discrete waypoints along the two paths.
This metric can be efficiently computed  via dynamic programming~\cite{agarwal2014computing,eiter1994computing}. 

\subsection{Trajectory-Optimization Approach}
\label{sec:trajopt_bug}
\begin{figure*}[!tb]
    \centering
		\begin{subfigure}[h]{0.33\textwidth}
        \centering
          \includegraphics[height=2.4cm]{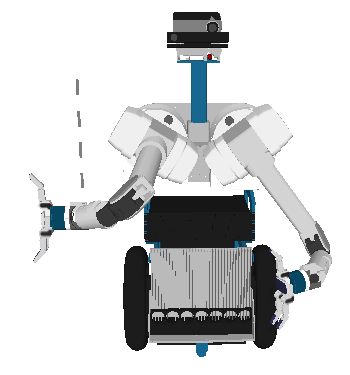}
        \includegraphics[height=2.4cm]{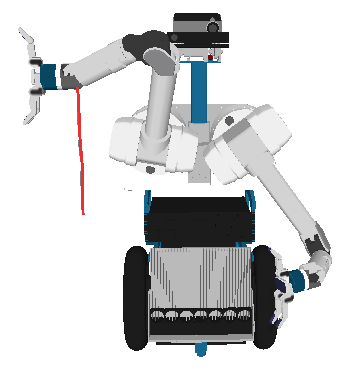}
          \caption{Successfully optimized path.}\label{fig:good}      
    \end{subfigure}
    \begin{subfigure}[h]{0.33\textwidth}
        \centering
          \includegraphics[height=2.4cm]{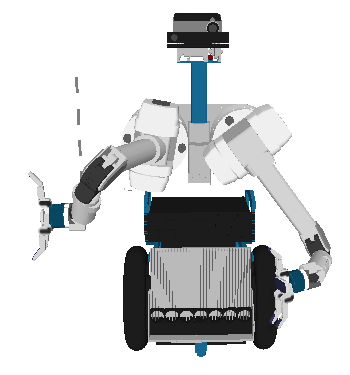}
        \includegraphics[height=2.4cm]{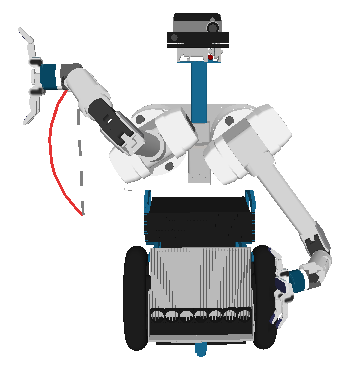}
          \caption{Unsuccessfully optimized path.}\label{fig:bad_start}      
    \end{subfigure}
    \begin{subfigure}[h]{0.16\textwidth}
        \includegraphics[height=2.4cm]{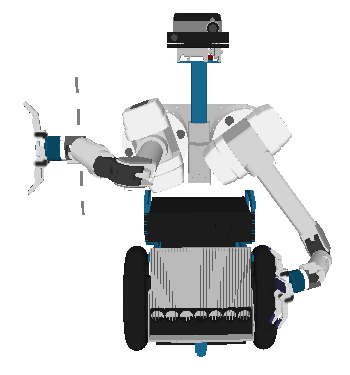}
        \caption{Splitting point.} \label{fig:bad_split}
    \end{subfigure}
    \begin{subfigure}[h]{0.16\textwidth}
        \includegraphics[height=2.4cm]{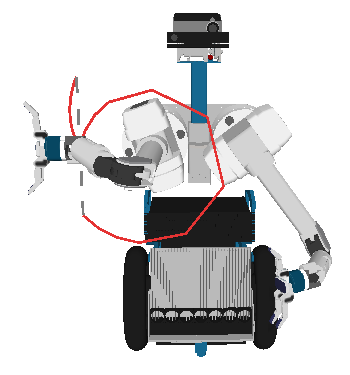}
        \subcaption{New path.} \label{fig:bad_end}
    \end{subfigure}
\caption{Visualization of the approach taken by Holladay et al.~\cite{holladay2016distance}. Reference path and computed path are shown in dotted black and solid red lines, respectively.}
\label{fig:trajopt}
\end{figure*}

The key insight from Holladay et al.~\cite{holladay2016distance} is to optimize Eq.~\ref{eqn:pdef} by minimizing $\frechetd(\bar{\xi}, \text{FK}(\xi))$. 
Framed as a trajectory-optimization problem, the paper provides methods to heuristically assist the optimizer by constraining the computed path into a sequence of smaller problems. 

We examine the algorithm's behavior on HERB, a bimanual mobile manipulator with seven degree-of-freedom arms~\cite{srinivasa2010herb}, as it tries to follow a straight-line reference path~$\bar{\xi}$, shown as the dotted line in~\figref{fig:trajopt}.
The algorithm picks start and end configurations and then plans a path from start to end, attempting to minimize $\frechetd(\bar{\xi}, \text{FK}(\xi))$. 



With the starting configuration in \figref{fig:good} (left), the planner drives the cost to zero, producing the solid red line path shown in~\figref{fig:good} (right).
However, since this is a redundant manipulator, the algorithm could have also picked the starting configuration shown in \figref{fig:bad_start} (left). 
Given this configuration, there is no path that exactly follows~$\bar{\xi}$. 
Therefore the optimizer produces the red path in \figref{fig:bad_start} (right), which deviates significantly from~$\bar{\xi}$.
The optimization-based algorithm of~\cite{holladay2016distance} will then split~$\bar{\xi}$ at the point where the generated path deviates the most from~$\bar{\xi}$, according to the \frechet distance.
In this case, it splits the path in the middle and samples an IK solution, shown in \figref{fig:bad_split}.
As shown in \figref{fig:bad_end}, the first half of the path still suffers from the original problem.

This limitation stems from the fact that the algorithm samples one IK solution for each pose along~$\bar{\xi}$.
However, there is a space of multiple IK solutions which may admit different paths.
This motivates our method, which searches over a space of IK solutions in an anytime fashion.

%% file: graph_creation.tex

\section{Generating a Set of Candidate Paths}
\label{sec:graph_creation}
Recall that our goal, defined in \eref{eqn:pdef}, is to find a collision-free path $\xi \in$ \Cspace such that $\text{FK}(\xi)$ minimizes the \frechet distance to $\bar{\xi}$.
This will be done by solving the following problem
\begin{equation}
	\xi^* = \arg \min_{\xi \in \Xi_{\bar{\xi}}} \frechetd( \text{FK}(\xi), \bar{\xi}),
\label{eqn:pdef2}
\end{equation}
and iteratively refining the number of waypoints along~$\xi$ and $\bar{\xi}$ to refine our  approximation of  the continuous \frechet distance. 
To do so, we build a layered graph~$L$ that approximates the set of candidate paths, $\Xi_{\bar{\xi}}$. 
As our algorithm progresses, we try to both improve the \textit{quality} of our path by exploring more candidate paths and improve the \textit{accuracy} of our \frechet approximation by increasing our sampling resolution.  

\subsection{Layered Graph}
\label{sec:layered_graph}
Consider the set of inverse kinematic solutions to all points along our reference path~$\bar{\xi}$:
\begin{equation}
\Man = \bigcup\limits_{\alpha \in [0, 1]} \text{IK}(\bar{\xi}(\alpha)). 
\label{eqn:cspace_ik}
\end{equation}
Any collision-free path that connects 
$\text{IK}(\bar{\xi}(0))$ 
with
$\text{IK}(\bar{\xi}(1))$ 
while completely lying on the manifold $\Man$ minimizes \eref{eqn:pdef}.
To approximate such a path, we sample $\Man$ and connect samples by straight-line segments in \cspaceWord (that may deviate from \Man).

The structure of \eref{eqn:cspace_ik} suggests two parameters that can be used to organize our sampling in a structured manner. 
The first is the location of a point in task space along $\bar{\xi}$, denoted by $\alpha$. 
The second is the set of IK solutions at each point.\footnote{Assuming that we have a redundant manipulator, there is an infinite set of IK solutions for each task-space point.} 
This further demonstrates the key limitation of the approach by Holladay et al.~\cite{holladay2016distance} which consider for each location $\alpha$ only one configuration.

\begin{wrapfigure}{r}{0.5\columnwidth}
\vspace{-5.5mm}
\centering
   \includegraphics[width=5cm]{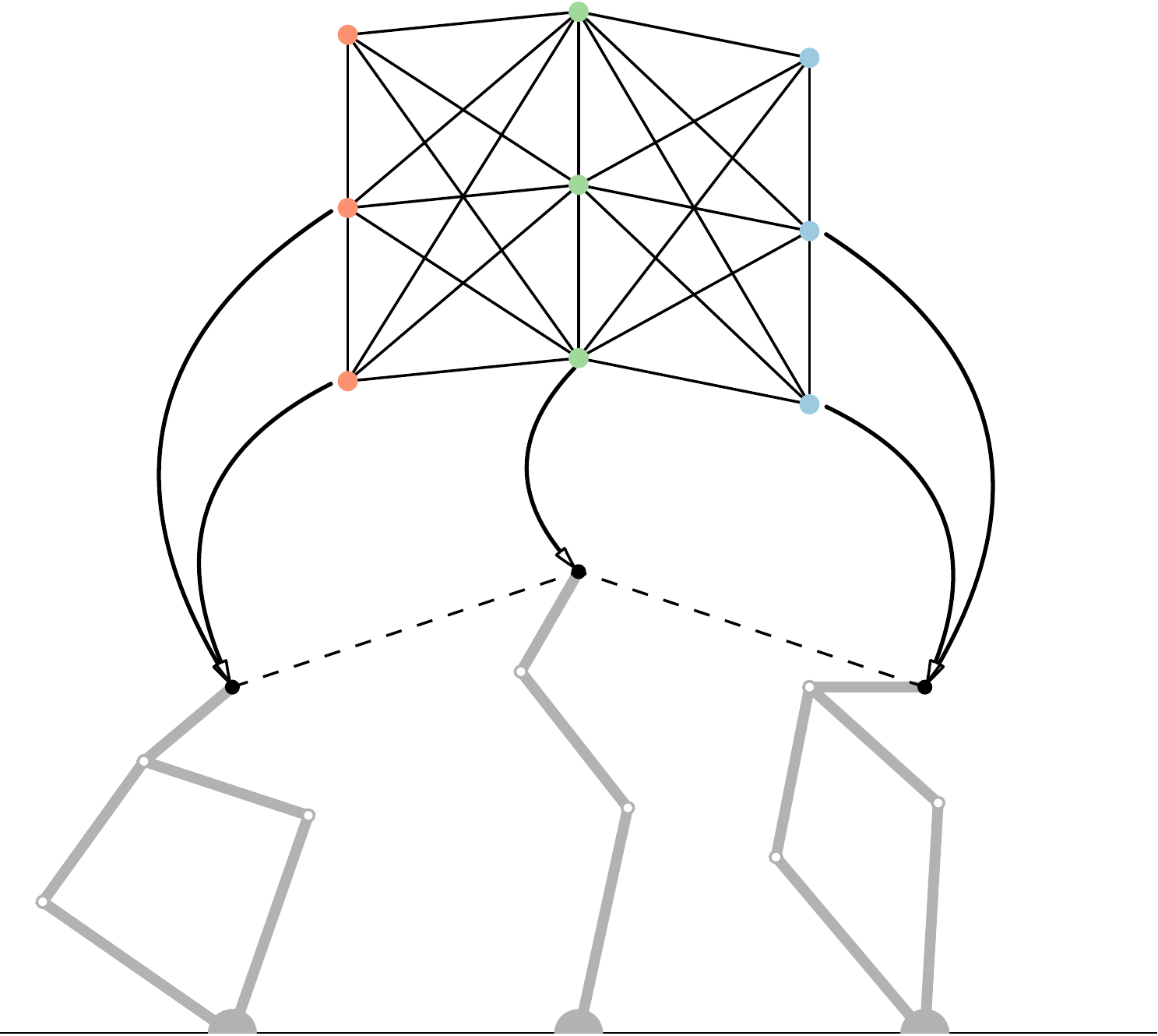}
   \caption{Each layer in our layered graph (top) maps to a task-space pose (bottom) along our reference path, shown as the dotted line. For each pose, there are multiple IK solutions, which make up the elements of each layer.}
  \vspace{-1.5mm}
\label{fig:layered_arms}
\end{wrapfigure}
Following the discussion above, we construct a layered graph $L = (V_{L}, E_{L})$ embedded in \cspaceWord where each layer is a set of IK solutions for one task-space point lying on the reference path~(\figref{fig:layered_arms}). 

To construct our graph, we begin by sampling $n$ waypoints in task space along our reference path $\{w_{1}\ldots w_{n}\}$. 
At each waypoint $w_{j}$, we initially compute up to~$k$ IK solutions $\{q_{j}^{1}\ldots q_{j}^{k}\}$ by querying random solutions from an analytical IK solver (IKFast~\cite{diankov2008openrave}).
Each configuration $q_{j}^{i}$ is a vertex in our graph~$L$. 
Namely, 
\begin{equation}
V_{L} = \{ q^{i}_{j} | 1 \leq j \leq n \text{ and } 1 \leq i \leq k \}.
\end{equation}

We next define our edge set, $E_{L}$.
Each vertex in a layer of IK solutions connects to every vertex in the subsequent layer and to every vertex in its own layer. 
Intuitively the path passes through every waypoint, with the freedom to select any IK solution for that waypoint. More formally,  
\begin{equation}
\begin{split}
E_{L} = &\{ (q_{j}^{i_{1}}, q_{j+1}^{i_{2}}) | 1 \leq j < n-1, 1 \leq i_{1}, i_{2} \leq k\} \\ 
        &\cup \{ (q_{j}^{i_{1}}, q_{j}^{i_{2}}) | 1 \leq j \leq n, 1 \leq i_{1}, i_{2} \leq k\}.
\label{eqn:layered_edge_definition}
\end{split}
\end{equation}

As mentioned, each edge is a straight-line segment in \cspaceWord between the two configurations, which does not necessarily lie on $\Man$.\footnote{In our implementation, we delay collision checking of edges along these paths.} To account for deviation from  $\Man$ in our discrete \frechet calculation, we subsample each edge. 
As we increase the subsampling resolution during our densification process, our discrete \frechet distance will better approximate the continuous \frechet distance~\cite{wylie2013discrete}.

\subsection{Na\"{i}ve Search Method}

Given our layered graph~$L$, let $\Xi_{L}$ denote the set of all paths in $L$ that connect any vertex in the first layer of $L$ to any vertex in the last layer of $L$.
We can restate \eref{eqn:pdef2} as
\begin{equation}
	\xi^*_L = \arg \min_{\xi \in \Xi_L} \frechetd( \text{FK}(\xi), \bar{\xi}).
\label{eqn:pdef3}
\end{equation}

Note that \frechet is a metric over entire paths, not path segments (i.e., individual edges), and thus we cannot simply search $L$ in a Dijkstra-like manner. 
One na\"{i}ve option would be to enumerate all candidate paths in $\Xi_{L}$ and compute $\frechetd(\bar{\xi}, \text{FK}(\xi_{L}))$ for all $\xi_{L}\in\Xi_{L}$. 
However, $|\Xi_{L}| =  \Omega({n^{k}})$. 
Instead, we adapt a method that computes the cross product between our layered graph and the reference path, allowing us to efficiently search for the minimal-cost path in $L$ in $O(n^2 k^2 \log (nk))$ time. 

%% file: search.tex

\section{Computing the Minimal-Cost Path}
\label{sec:cross_product_search}
To efficiently compute a solution to \eref{eqn:pdef3} we represent $\bar{\xi}$ as a (one-dimensional) graph $G_{\bar{\xi}} = (V_{\bar{\xi}}, E_{\bar{\xi}})$ where $V_{\bar{\xi}}$ are the sampled waypoints of $\bar{\xi}$ and an edge $e \in E_{\bar{\xi}}$ connects two subsequent waypoints. 
This allows us to view both $L$ and~$G_{\bar{\xi}}$ as abstract one-dimensional simplicial complexes\footnote{An abstract simplicial complex is a combinatorial description of a simplicial complex---a 
set composed of points, line segments, triangles, and their n-dimensional counterparts~\cite{munkres1984elements}.}.
Har-Peled and Raichel introduced an algorithm for computing the \frechet distance between two such complexes by considering their cross product~\cite{har2014frechet}.
Therefore, our instance is a restricted case of their problem and we present our adaptation.
Following \figref{fig:fig1}, we first create a new graph
$\Phi = L \times G_{\bar{\xi}}$  
and then use it to solve \eref{eqn:pdef3}.

\subsection[Cross-Product Graph]{Cross-Product Graph~$\Phi$}
\label{sec:cp_graph}
In this section, we define for $\Phi$ the set of vertices~$V_\Phi$, edges~$E_\Phi$ and their costs.
Set~$V_\Phi = V_{\bar{\xi}} \times V_L$.
Namely, each vertex in $V_\Phi$ is a pair $(w, q)$ such that $w \in V_{\bar{\xi}} $ and $q \in V_L$.
An edge connects two vertices in $V_L$ if either or both elements of each vertex are adjacent to each other in their respective graph.
Namely,
\begin{equation}
\begin{split}
    E_{\Phi} = &\{ ((w_{m_{1}}, q^{i_{1}}_{j_{1}}), (w_{m_{2}}, q^{i_{2}}_{j_{2}})) | \text{ if} \\
& ((w_{m_{1}} = w_{m_{2}}) \text{ and } (q^{i_{1}}_{j_{1}}, q^{i_{2}}_{j_{2}}) \in E_{L}) \text{ or } \\
& ((w_{m_{1}}, w_{m_{2}}) \in E_{\bar{\xi}} \text{ and } (q^{i_{1}}_{j_{1}} = q^{i_{2}}_{j_{2}})) \text{ or } \\
& ((w_{m_{1}}, w_{m_{2}}) \in E_{\bar{\xi}} \text{ and } (q^{i_{1}}_{j_{1}}, q^{i_{2}}_{j_{2}}) \in E_{L} ) \}.
\label{eqn:cross_product_edges}
\end{split}
\end{equation}

Set $\text{cost}(w,q) = d_{TS}(w, \text{FK}(q))$ 
to be the cost\footnote{Har-Peled and Raichel~\cite{har2014frechet} use the term ``elevation'' to refer to our notion of cost.} of a vertex $(w, q) \in V_\Phi$.
The cost of an edge $e = (u,v)$ is the maximum of the cost of its endpoints. 
Namely, $\text{cost}(u,v) = \max (\text{cost}(u), \text{cost}(v))$.


The cost of a path in $\Phi$ is defined as the \emph{maximal} edge cost along this path,
also known as the ``bottleneck cost''. Har-Peled and Raichel show that the cost  of such a path is equal to the \frechet distance between the corresponding curves in the two simplicial complexes that compose the product space~\cite{har2014frechet}. 
In other words, the cost of a path in the cross-product graph $(w_{1}, q_{1})\ldots (w_{n}, q_{n})$ corresponds to the discrete \frechet distance between the discretized reference path
$\{w_{1}\ldots w_{n}\}$ 
 and the FK of the polyline 
$\{q_{1}\ldots q_{n}\}$ in the layered graph. 
Thus, our goal can be restated as finding the bottleneck shortest path in $\Phi$.

\subsection{Computing the Bottleneck Shortest path}
Computing the bottleneck shortest path in a graph $G = (V,E)$ is a well-studied problem and there are efficient algorithms that run in time linear in $|E|$~\cite{har2014frechet}.
However, we chose to use a simple variant of Dijkstra's algorithm~\cite{dijkstra1959note} (with complexity $\bigoh{|E_{\Phi}|+|V_{\Phi}|\log|V_{\Phi}|}$) because we have found it to be empirically faster and it allows for more efficient updates to $\Phi$, as described in \sref{sec:densification}.

Standard implementations of Dijkstra's algorithm assume that the cost of a path to vertex $v$ coming from vertex $u$ is the cost to reach $u$ plus the cost of the edge~$(u, v)$. 
To compute the bottleneck cost, we simply swap the sum of costs for a $\max()$ such that the cost$(v)$ = $\max($cost$(v)$, cost$(u, v))$. 

Given the bottleneck shortest path in $\Phi$, we can extract the corresponding path in the layered graph to generate~$\xi_{L}$ that optimizes \eref{eqn:pdef3}.
While searching, we lazily evaluate the edges in $\xi_{L}$ for collision~\cite{dellin2016unifying,HMPSS18}. 

%% file: densification.tex

\section{Densification}
\label{sec:densification}

Following the construction of the cross-product graph $\Phi$, we want to iteratively improve 
(i)~the quality of our solution and 
(ii)~the accuracy of our approximation of the continuous \frechet distance. 

To improve the quality of our solution, we \textit{densify} our layered graph $L$ to provide more candidate paths to search over.
The two parameters of our layered graph (number of layers and number of IK solutions in each layer) suggest two approaches: we can either add another layer to~$L$ by choosing a new waypoint along $\bar{\xi}$ and sampling $k$ IK solutions of this waypoint or we can increase the size of an existing layer in~$L$ by sampling more inverse kinematic solutions at an existing waypoint. 
To improve our approximation of the continuous \frechet distance, we increase the subsampling resolution along edges of our two structures,~$G_{\bar{\xi}}$ and $L$. 
Given updates to $G_{\bar{\xi}}$ or~$L$, we then update our cross-product graph $\Phi$ accordingly. 

To summarize, given these two objectives we have defined three densification methods: 
(i)~adding a layer to~$L$, 
(ii)~adding IK samples to an existing layer of $L$ and 
(iii)~increasing subsampling resolution of an edge of $L$ or $G_{\bar{\xi}}$. 
Given these three densification methods, we present several strategies on how to apply them followed by experimental comparisons.

\subsection{Densification Strategies}
\label{sec:densification_strategies}
Our strategies on where to apply our densification methods are inspired by the PRM literature, which balance global and local updates~\cite{kavraki1994randomized,kavraki1996probabilistic,bohlin2000path}. 

Global updates sample either $L$ or $G_{\bar{\xi}}$ uniformly to determine where to add a layer, which layer to augment or which edge to increase the sampling resolution of. 
Local updates are applied in the neighborhood of the \textit{bottleneck node} along the current best path in $\Phi$, where the \frechet distance is the largest. 
From the \frechet distance analogy, this is where we have the longest leash between a point in the layered graph $L$ and a point in the reference graph $G_{\bar{\xi}}$\footnote{This is similar to the stapling method described in \cite{holladay2016distance} in that both leverage the \frechet distance to heuristically focus effort to improve the quality of the current solution.}.

Within a single step of densification, we first determine whether to use local or global updates and then pick a densification method uniformly.
Having densified either $G_{\bar{\xi}}$ or $L$ and updated $\Phi$ accordingly, we search $\Phi$ for best current solution. 
Our Dijkstra-like search of~$\Phi$ is thus an instance of a dynamic shortest-path problem which allows us to use efficient algorithms that reuse information from previous search episodes~\cite{frigioni2000fully,ramalingam1996computational,koenig2004lifelong}. 
This loop is illustrated in \figref{fig:fig1}. 

We present two strategies for determining whether to use global or local updates and proceed to empirically evaluate their performance. 

\begin{description}
	\item[Hybrid Strategy] trades off between local and global updates by choosing local updates with probability~$p$. The values of $p=0$ and $p=1$ correspond to purely local updates and purely global updates, respectively.
	\item[Local-then-Global Strategy] combines local and global methods by reasoning about the progress made across multiple densification steps.  
We use local updates as long as they continue to improve the current best solution. 
Once~$m$ successive iterations of local updates do not decrease the bottleneck cost, we switch to performing global updates. 
If global updates reduce our cost, this strategy returns to applying local updates.
\end{description}

\subsection{Experimental Comparison of Densification Strategies}
\label{sec:experiment_densification}
To compare densification strategies, we use the bimanual manipulator HERB to generate 100 instances of layered graphs for a given reference path $\bar{\xi}$, all with the same initial number of waypoints, IK solutions per waypoint, and level of subsampling resolution. 
For each problem we randomly place rectangular boxes in the vicinity of the robot. 
We then conduct many iterations of densification.
We repeat this process for multiple parameter settings and reference paths. 
While a summary is given below, there are more detailed experimental results available in~\textVersion{the extended version of this paper~\cite{holladay2017minimizing}}{Appendix~\ref{apndx:results}}.

Our two strategies, hybrid and local-then-global, each have one parameter.
We compare discrete choices of the parameters to select the best one. 
For the hybrid strategy we compare $p$-values in the set $\{0, 0.25, 0.5, 0.75, 1\}$ and observe that lower $p$-values (namely, biasing local updates), produce paths with a shorter \frechet distance at each iteration. 
For the local-then-global strategy (referred to as, L-t-G) we compare $m$-values in the set $\{2, 3, 4, 5, 6\}$ and observe that mid-range $m$-values produce the best-quality results.
Therefore, in comparing our two densification strategies, we used $p=0.25$ for the hybrid strategy and $m=5$ for the local-then-global strategy (\figref{fig:all_methods}). 

These results indicate that, while the \frechet distance is a metric over entire paths and global updates are required, it is benficial to heuristically guide the densification process in the neighborhood of the local bottleneck. 

For both strategies, most of the computation time is spent collision checking the path segments, with some smaller fraction spent computing the cost of nodes in the cross-product graph.  
Before empirically comparing our method to alternative algorithms, we first provide a proof sketch of its asymptotic optimality. 

%% file: proof.tex
\section{Asymptotic Optimality}
\label{sec:proof}

In this section we state our main theoretic result and provide a proof outline.
We show that, under some technical assumptions, our algorithm is asymptotically optimal.
To do so, we assume that \Man, the set of all configurations that directly map to~\Rpath, contains a ``well-behaved'' portion.
This notion, together with the proof are detailed in \textVersion{the extended version of this paper~\cite{holladay2017minimizing} as well as in the supplementary material}{Appendix~\ref{apndx:proof}}.

\begin{thm}
\label{thm:ao}
If $\Man$ contains a well-behaved portion~\WellMan  then
our algorithm is asymptotically optimal.
Namely, as $n\rightarrow \infty$ and $k\rightarrow\infty$ it will asymptotically find a collision-free \cspaceWord path whose \frechet distance from $\Rpath$ tends to zero.	
\end{thm}

\subsection{Proof sketch}
Our proof relies on certain properties (detailed in \textVersion{the extended version of this paper~\cite{holladay2017minimizing} as well as in the supplementary material}{Appendix~\ref{apndx:proof}}) which hold for any redundant manipulator.
Roughly speaking, we require that there is some  correlation between distances in \cspaceWord and distances in task space. 
This is required to ensure both that (i)~connecting close-by samples on~$\Man$  will lead to minimizing the \frechet distance and that (ii)~sampling close-by points in task space can yield close-by configurations, given enough IK solutions.  

Our proof sketch is similar in nature  to~\cite[Thm.~34]{karaman2011sampling}. 
We assume that there exists some path~$\xi$ lying on~\WellMan that directly maps to \Rpath.
Namely, we have that $\frechetd\left(\text{FK}(\xi), \Rpath\right) = 0$.
We will show that there exists a sequence of a family of paths $\{ \Xi_n\}_{n \in \mathbb{N}}$ such that any sequence of paths $\{\xi_n \in \Xi_n\}_{n \in \mathbb{N}}$ converge  to $\xi$ (recall that~$n$ is the number of sampled waypoints along the reference path $\Rpath$).
For each path $\xi_n \in \Xi_n$ we show that there exists some $\varepsilon_n$ such that
$$ \frechetd{(\text{FK}(\xi_n), \Rpath)} \leq \varepsilon_n $$
and that 
$$ \lim_{n \rightarrow \infty} \varepsilon_n = 0. $$
Furthermore, if $P_n$ is the probability that our algorithm will produce a collision-free path in $\Xi_n$ then we will show that 
\begin{equation*}
\lim_{n \rightarrow \infty} P_n = 1.
\end{equation*}
Combining the above will yield that our algorithm is asymptotically optimal.
For a figure depicting the notation used throughout our proof, see~\figref{fig:proof}.

\begin{figure}
	\centering
	\includegraphics[width=0.99\columnwidth]{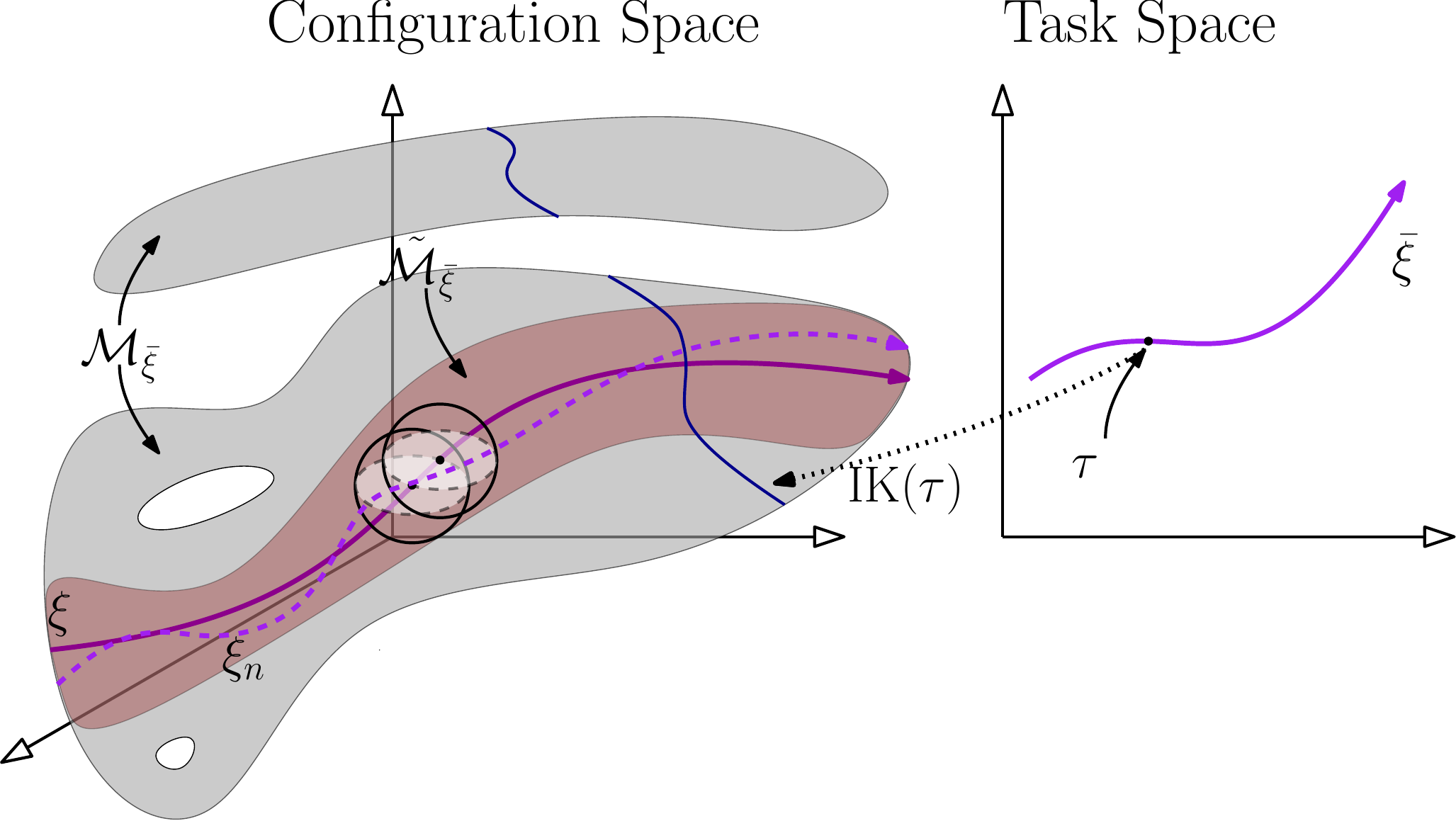}
	\caption{Visualization of the notation used in proof sketch.}
	\label{fig:proof}
\end{figure}

\subsection{Discussion}
One assumption that we take implies that there is a path in \cspaceWord that directly maps to the reference path.
This is due to our sampling scheme which requires that we only sample \emph{on} the reference path and not around it.
If this assumption does not hold, an algorithm that minimizes the \frechet distance cannot restrict itself to sampling only on the reference path.

An interesting difference between our proof and existing asymptotic-optimality proofs such as~\cite[Thm.~34]{karaman2011sampling} is that our algorithm connects any two vertices in subsequent layers. 
Thus, we did not have to argue about connection radius but about the rate at which we sample waypoints and IK solutions. It would be interesting to alter the algorithm to only connect vertices in subsequent layers if they are within some predefined distance. This would require adding this constraint to the proof of Thm~.\ref{thm:ao}.

%% file: results.tex

\section{Experimental Results}
\label{sec:results}

\begin{figure*}[!t]
    \centering
    \minipage{0.32\textwidth}
        \includegraphics[width=\textwidth]{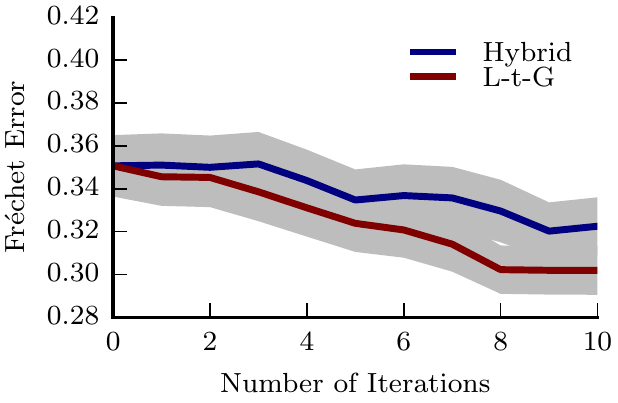}
        \subcaption{Densification Strategies} \label{fig:all_methods}
    \endminipage 
    \minipage{0.32\textwidth}
        \includegraphics[width=\textwidth]{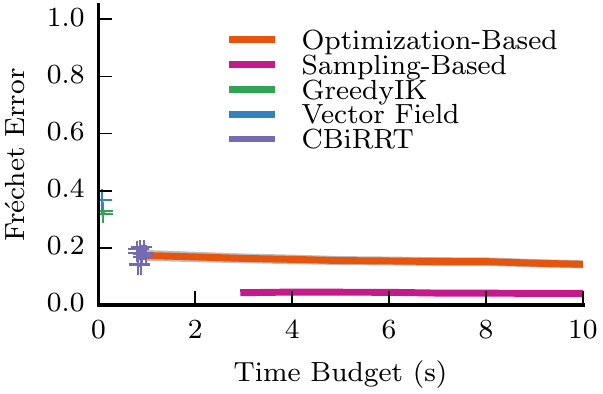} 
        \subcaption{Straight-line path} \label{fig:line_compare}
    \endminipage
    \minipage{0.32\textwidth}
        \includegraphics[width=\textwidth]{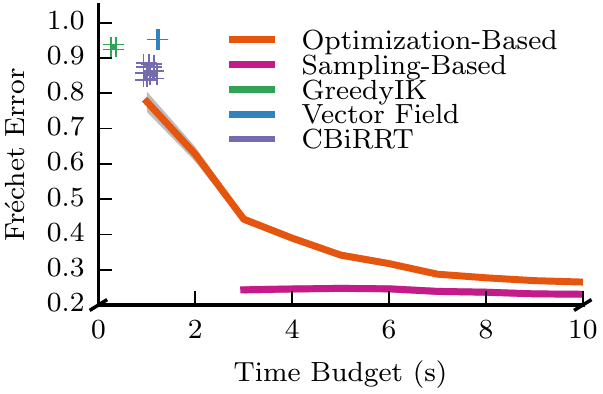} 
        \subcaption{Random path} \label{fig:algo_compare}
    \endminipage
\caption{Empirical evaluation.
	(a)~A comparison of our densification strategies.
	(b,c)~A comparison of our algorithm with state-of-the-art planners on a straight-line path and a random path, respectively.
	While each figure only shows the results for one reference path and initial layered graph sizes, repeated experiments showed these results were consistent across multiple reference paths and graph sizes. }
\vspace{2mm}
\end{figure*} 

We compare our sampling-based algorithm with four other planners: an optimization-based approach~\cite{holladay2016distance}, a vector-field planner~\cite{srinivasa2016system}, a greedy inverse kinematic planner~\cite{srinivasa2016system} and CBiRRT (Constrained Bidirectional RRT)~\cite{berenson2009manipulation}. 

The optimization-based algorithm from \cite{holladay2016distance}, summarized in \sref{sec:trajopt_bug}, continues to split the path into subproblems until the \frechet distance between the entire path and the reference path is below some threshold value\footnote{In \cite{holladay2016distance} this is referred to as "stapling in task space".}. 
We adapt this to an anytime algorithm where an entire path is produced and evaluated after each split.
The vector-field planner integrates a Jacobian pseudo-inverse to follow a vector field defined by our path~\cite{srinivasa2016system}.
The greedy inverse kinematic planner (GreedyIK) samples IK solutions from $\bar{\xi}$ and attempts to interpolate between them in \cspaceWord~\cite{srinivasa2016system}. 
CBiRRT plans on constraint manifolds by projecting random samples to our manifold $M_{\bar{\xi}}$~\cite{berenson2009manipulation}. 
The algorithm is set to only accept projected samples if they fall within some threshold distance $\kappa$ of any point on the reference path.

We use the same experimental setup as described in \sref{sec:experiment_densification}, averaging the results of each planner over 100 instances. 
The sampling-based and optimization-based planners both have anytime performance, so we query each planner after $t$ seconds for their best solution so far.
Vector Field, GreedyIK and CBiRRT are treated as single-query planners and therefore do not have anytime performance. 
For CBiRRT, we plot the performance across several thresholds (we use $\kappa$-values in the set $\{0.2, 0.3, 0.4, 0.5\}$). 

We first compare performance of all of the planners for a simple, straight-line path in \figref{fig:line_compare}. As expected, each of the planners does fairly well, with our sampling-based algorithm producing paths closest to the reference path. 

We then test on a variety of paths, some of which are shown in various colors in \figref{fig:many_path}.
As a representative example, we compare performance in \figref{fig:algo_compare} and the progression of the anytime algorithms in \figref{fig:algo_progression} for one particular path. 
Further performance comparisons are available in~\textVersion{the extended version of this paper~\cite{holladay2017minimizing}}{Appendix~\ref{apndx:results}}.

For this path, the single-query planners quickly produce solutions of low quality. CBiRRT produces equally low-quality solution because the projection and threshold constraints do not enforce the monotonicity of traversal that \frechet distance does for our sampling-based algorithm. 

\begin{wrapfigure}{r}{0.5\columnwidth}
\vspace{-5.5mm}
\centering
   \includegraphics[width=4cm]{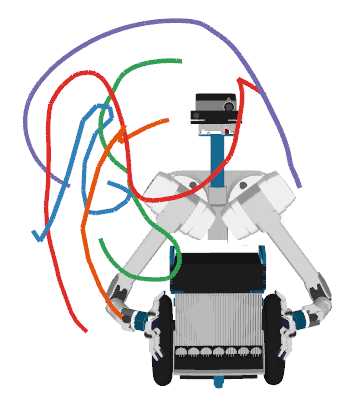}
   \caption{Each color shows one of the many reference paths we evaluated the algorithms on.}
  \vspace{-1.5mm}
\label{fig:many_path}
\end{wrapfigure}
Turning to our anytime algorithms, the optimization-based approach finds an initial solution faster, but its solution is significantly worse than the one found by the sampling-based approach. 
While the optimization-based approach improves its solution at a faster rate, the sampling-based approach produces a higher-quality path for a fixed time budget. 


For both examples our sample-based approach is able to converge to a path that more closely follows the reference path because it searches over sets of IK solutions and leverages the \frechet distance to efficiently search.
It is important to note that, as expected, the quality of the solution obtained by our planner does slightly decrease over time. It is hard to observe this trend in Fig. 5c due to the scale needed to compare with the other planners but this can be observed in Fig. 5a as well as in the \textVersion{extended version of this paper~\cite{holladay2017minimizing}}{Appendix~\ref{apndx:results}}.

\begin{figure}[!tb]
    \centering
    \minipage{0.33\columnwidth}
        \includegraphics[width=\textwidth]{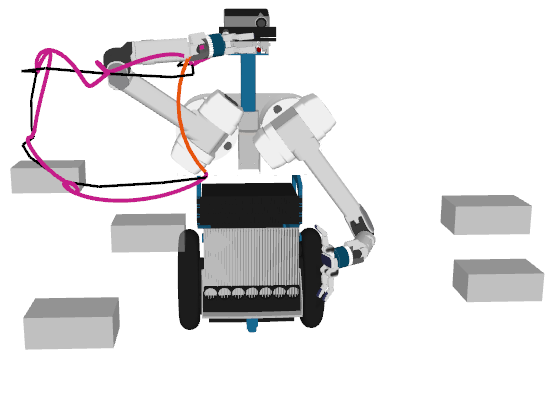}
        \subcaption{Initial Paths} \label{fig:initial_herb}
    \endminipage 
    \minipage{0.33\columnwidth}
        \includegraphics[width=\textwidth]{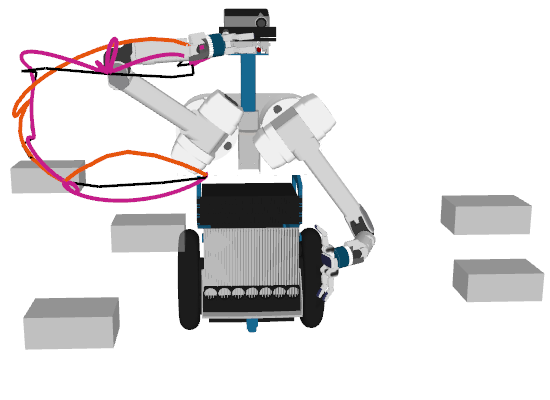}
        \subcaption{Midway Progress} \label{fig:mid_herb}
    \endminipage 
    \minipage{0.33\columnwidth}
        \includegraphics[width=\textwidth]{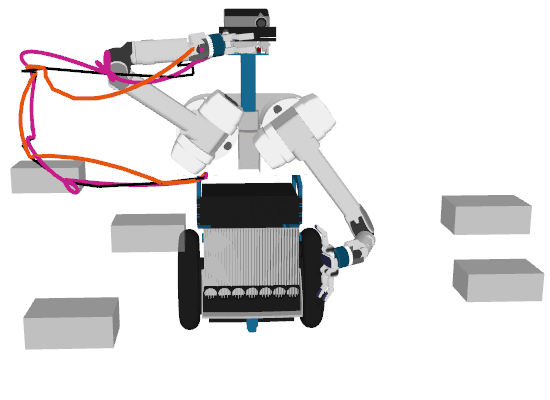}
        \subcaption{Final Paths} \label{fig:final_herb}
    \endminipage 
\caption{We show the progression of the optimization-based approach (orange) and the sample-based approach (pink) as they try to follow the reference path (black). Randomly generated obstacles in the environment are shown in grey. These figures only capture the differences in position, not orientation.}
\label{fig:algo_progression}
\vspace{2mm}
\end{figure}

%% file: discussion.tex

\section{Future Directions and Discussion}
\label{sec:discussion}
We presented an anytime algorithm that produces a collision-free configuration space path that "follows", according to the \frechet distance, a reference path in task space.
By leveraging the \frechet distance, we were able to organize our space of candidate solutions into a structure that we could efficiently search and incrementally densify.
We outlined several strategies for densifying the structure and provided a proof sketch of asymptotic optimality.
We concluded with a comparison of our algorithm against several state-of-the-art planners across multiple paths and parameter settings. 
Looking forward, we present several future research directions.

In this work we considered how to optimize paths to follow a reference path. 
We did not consider the length of the path in \cspaceWord.
In the future, we wish to focus on the bicriteria optimization problem of simultaneously decreasing both the distance (in task space) from the reference path and the path length (in \cspaceWord). 

In addition to task space positions, we may also want to specify end-effector velocities~\cite{maciejewski1985obstacle} or forces. 
Another variant, originally suggested by \cite{oriolo2002probabilistic} and explored with the Procrustes distance metric in \cite{holladay2016distance}, is to consider paths without a fixed starting point, thus allowing the shape to be translated and rotated in space freely. 

Our algorithm only samples IK solutions on the reference path. 
Given many obstacles, we may want to encourage our path to deviate slightly by sampling solutions from an $\varepsilon$-ball around our reference path. 
This additional flexibility would require revisiting our theoretic guarantees on asymptotic optimality.

Finally, our work draws some parallels to Hauser's recent work on global redundancy resolution~\cite{HE18}.  
Both algorithms create PRM-like structures in configuration space, but our work has focused on pathwise redundancy resolution~\cite{HE18} and we leverage a different search method and optimization criteria. 
We believe that our work, which is complementary to his, may benefit by using his method to pick good IK solutions and leave this as an area of future work.


%% file: appendix_proof.tex
\section{Asymptotic Optimality (AO) Proof}
\label{apndx:proof}
In this section we provide our main theoretic result.
For completeness, we repeat the main theorem, stated in Sec. 6.

\subsection{Preliminaries}

Before we can formally state our main theorem, we need to detail some technical assumptions.
For a figure visualizing the different notations used in this section, see Fig.~\ref{fig:proof-full}.
Recall that~$\Rpath$ denotes the task-space reference path and let $\Man = \{ q \in \Cspace \ | \ \text{FK}(q) \in \Rpath\}$ denote the set of all configurations that directly map to~\Rpath.
In our proof we will rely on the assumption that $\Man$ contains a ``well-behaved'' portion, which we now define:

\subsubsection{Well-behaved portions of $\calM$} 
A ``well-behaved'' portion of $\Man$, denoted as~$\WellMan$, is a region that adheres to the following assumptions:

\begin{enumerate}[label={\textbf{A\arabic*}}]
	\item The region~$\WellMan$ is a manifold of dimension~$\text{dim}({\Man})$
(namely, there are no low-dimensional singularities on~$\WellMan$). 
This is required to ensure that our sampling-based approach does not have to sample within zero-measure submanifolds of $\Man$.

	\item There exists some $\delta>0$ such that every path lying on~$\WellMan$ has clearance~$\delta$ from the closest obstacle. 
This is required to ensure that a \cspaceWord path mapping to our reference path exists and, again,  that it can be sampled.

	\item The region~$\WellMan$ is connected and $\exists q_0, q_1\in \WellMan$ such that $\text{FK}(q_0) = \Rpath(0)$ and $\text{FK}(q_1) = \Rpath(1)$.
Namely, there are configurations that map to the start and end of the reference path that lie on~$\WellMan$ and there exists some path connecting the two.
This is required to ensure that there is a path lying on~$\WellMan$ that maps to our reference path.

	\item There exists some function $\eta_4(\varepsilon)$ such that for every~$\varepsilon$ and for every~$q  \in \WellMan$ it holds  that
$$
| \{ \text{FK}(q') \mid q' \in \Man \text{ and }\distC(q,q') \leq \varepsilon
\}| / |\bar{\xi} | = \eta_4(\varepsilon) 
$$
Namely, any ball lying on the well-behaved manifold is mapped to a non-negligible portion of the reference path~$\Rpath$.
Specifically, we assume that there exists some $p>0$ such that $\eta_4(\varepsilon) = \omega(\varepsilon^p)$.

\end{enumerate}

\subsubsection{Redundant manipulator properties}
Our proof will also rely on certain properties which hold for any redundant manipulator. Specifically,

\begin{enumerate}[label={\textbf{P\arabic*}}]
	\item For every $\varepsilon$, there exists some $\eta_1(\varepsilon)$  where $\lim_{\varepsilon \rightarrow 0} \eta_1(\varepsilon) = 0$ such that
	$\forall q,q' \in \Cspace \text{ s.t. } \distC(q,q') \leq \varepsilon$, it holds that,
$$
\distT(\text{FK}(q), \text{FK}(q')) \leq \eta_1(\varepsilon).
$$
Namely, close-by configurations map to close-by points in task space.

	\item For every $\varepsilon$, there exists some $\eta_2(\varepsilon)$ such that
$\forall 
\tau \in \bar{\xi}$ and 
$\forall 
q \in \text{IK}(\tau)$,
$$
|\{
	q'\in \text{IK}(\tau)~s.t.~\distC(q,q') \leq \varepsilon 
\}| / |\{ q'\in \text{IK}(\tau) \}| \geq \eta_2(\varepsilon).
$$
Namely, the portion of the self manifold~\cite{burdick1989inverse} mapped close to a given configuration is a non-negligible portion.
Specifically, let 
$\ell_{\text{SM}}$ be the total variation (length) of the longest self manifold among all task-space points along the reference path.
Thus, we have that,
$
\eta_2(\varepsilon) = \varepsilon / \ell_{\text{SM}}
$.
	
\end{enumerate}

\subsubsection{Distances:}
Finally, for completeness recall that 
$\distT$ and $\distC$ denote our distance metrics in the Task space and configuration space, respectively.

\subsection{Theorem statement}

\begin{figure}
	\centering
	\includegraphics[width=0.99\columnwidth]{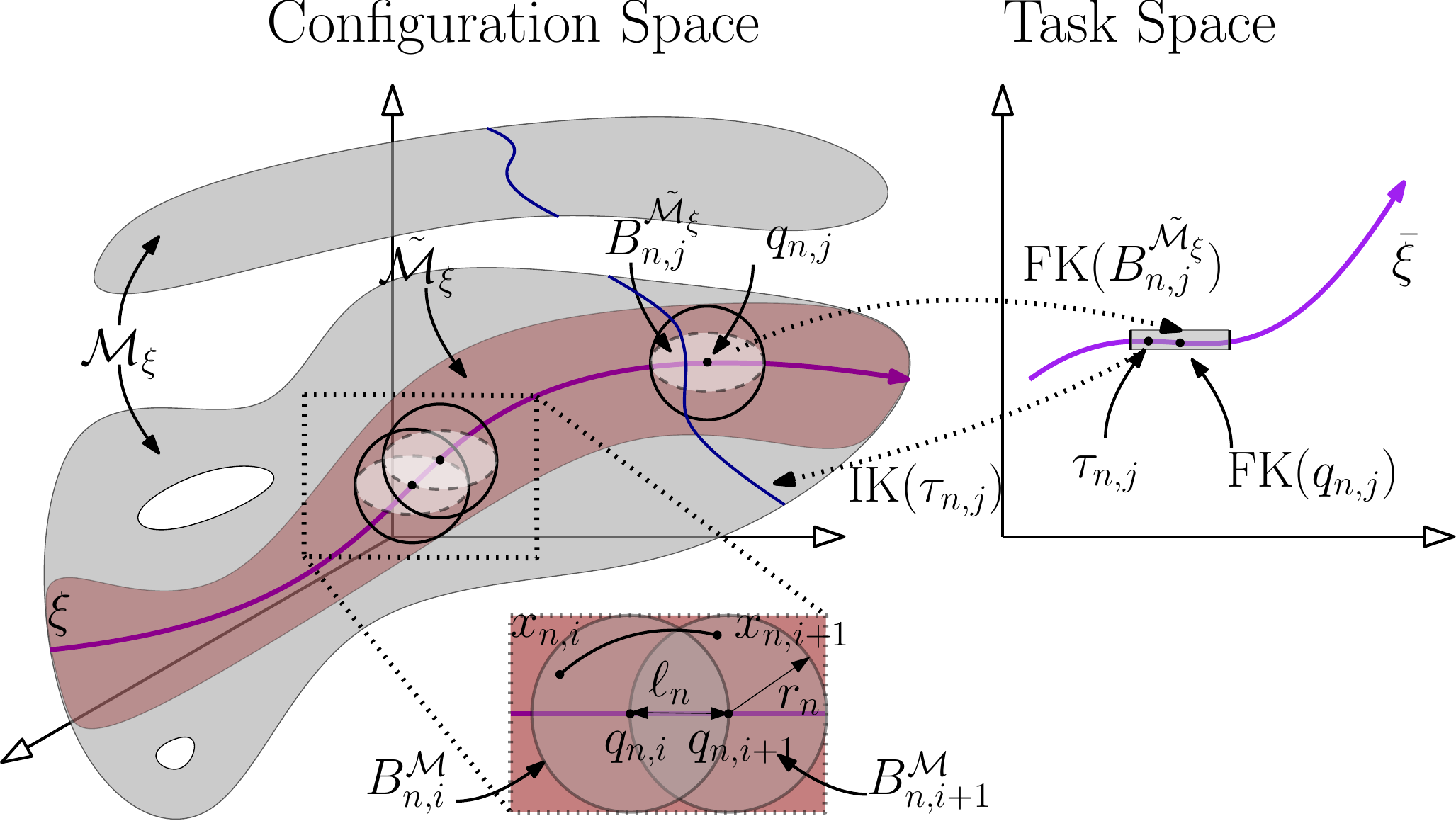}
	\caption{Visualization of the notation used in proof.}
	\vspace{2mm}
	\label{fig:proof-full}
\end{figure}

In this section, we restate our theorem and provide an outline of the proof.

\begin{thm}
\label{thm:ao}
If $\Man$ contains a well-behaved portion then
our algorithm is asymptotically optimal.
Namely, as $n\rightarrow \infty$ and $k\rightarrow\infty$ it will asymptotically find a collision-free \cspaceWord path whose \frechet distance from~$\Rpath$ tends to zero.	
\end{thm}

\subsection[The Family of Paths]{The Family of Paths $\Xi_n$}
\label{sec:seq}
For each $n \in \mathbb{N}$ we construct a set $B_n = \{B_{n,1}, B_{n,2}, \ldots, B_{n,M_n} \}$ of $M_n$ overlapping balls in \cspaceWord, each with radius $r_n$ that collectively ``cover'' our chosen path $\xi$.
Let $q_{n,i} \in \xi$ denote the center of the $i$'th ball $B_{n,i}$.
The sequence $B_n$ is defined such that:
(i)~the center $q_{n,0}$ of $B_{n,0}$ is $\xi(0)$,
(ii)~the center $q_{n,M_N}$ of the last ball $B_{n,M_n}$ is $\xi(1)$,
(iii)~the distance between the centers of two consecutive balls (except, possibly the last) is exactly $\ell_n = r_n$ 
and that
(iv)~$\lim_{n \rightarrow \infty } r_n = 0$.

Set $B_{n,i}^{\Man} = B_{n,i} \cap \WellMan$
and set
$X_n = \{x_{n,1}, x_{n,2}, \ldots, x_{n,M_n} \}$ to be a sequence of configurations such that $\forall i \ x_{n,i} \in B_{n,i}^{\Man}$.
Namely, the $i$'th  configuration lies on the intersection of~$\WellMan$ and the $i$'th ball $B_{n,i}$.
Each such sequence~$X_n$ induces a path $\xi_n \in \Cspace$ connecting consecutive points of $X_n$.
Let~$\calX_n$ be the set of all such sequences and~$\Xi_n$ the set of all such paths.

\begin{lem}
\label{lem:collision-free}
	$\exists n_0$ s.t. 
	$\forall \xi_n \in \Xi_n$
	it holds that $\xi_n$ is collision free  for $n \geq n_0$.
\end{lem}
\begin{proof}
As $\lim_{n \rightarrow \infty } r_n = 0$, there exists some $n_0$ such that 
$\forall n \geq n_0$ it holds that $r_n \leq \delta/2$. 
Thus, given any two configurations $x, x'$ that lie in consecutive balls it holds that the straight-line segment connecting~$x$ and~$x'$ is collision free.
This, in turn implies that 
$\forall \xi_n \in \Xi_n$, $\xi_n$ is collision free when $n \geq n_0$.
\end{proof}

\ignore{

 and such that 
given any two configurations $x, x'$ that lie in consecutive balls it holds that
(i)~$\frechetd(\text{FK}{(x)},\text{FK}{(x')}) \leq \varepsilon_n$ 
and
(ii)~the straight-line segment connecting~$x$ and~$x'$ is collision free.
These requirements will be satisfied by choosing appropriate values of $r_n$ and $M_n$.

Let 
$\beta \in (0,1)$ be some constant and 
set
$\ell_n = \beta r_n$.
Furthermore, recall that $\xi$ is a path lying on $\WellMan$ and let 
$\ell:=\text{TV}(\xi)$ denote it's total variation (length).
Note that the distance $\ell_n$ between two consecutive centers of balls in $B_n$ is bounded by 
$$
	\ell_n \leq \ell / M_n.
$$

Now, consider any two configurations $x_{n,i}, x_{n,i+1}$ that lie in consecutive balls $B_{n,i}$ and $B_{n,i+1}$, respectively.
It holds that:\\
(i)~the distance in \Cspace between $x_{n,i}$ and $x_{n,i+1}$ is bounded.
Specifically, using the triangle inequality and the fact that $\ell_n \leq r_n $,
\begin{align}
\label{eq:cspace-distances}
 \begin{split}
	\distC(x_{n,i}, x_{n,i+1}) & \leq   \distC(x_{n,i}, q_{n,i}) + \distC(q_{n,i}, q_{n,i+1}) + \distC(q_{n,i+1}, x_{n,i+1}) \\
	& \leq 2 r_n + \ell_n \\
	& \leq 3 r_n
 \end{split}
\end{align}

\noindent
(ii)~the distance in task space between FK$(x_{n,i})$ and FK$(x_{n,i+1})$ is bounded.
Specifically, using \emph{P1} and Eq.~\ref{eq:cspace-distances},
\begin{align}
 \begin{split}
	\distT(\text{FK}(x_{n,i}), \text{FK}(x_{n,i+1})) & \leq  \eta(\distC(x_{n,i}, x_{n,i+1}))\\
	& \leq \eta_1(3 r_n).
 \end{split}
\end{align}
}

\subsection[Sampling a Sequence]{Sampling a Sequence $X_n$}
In Section~\ref{sec:seq} we defined the sequence~$X_n = \{x_{n,1}, x_{n,2}, \ldots, x_{n,M_n} \}$ and showed that the induced path~$\xi_n$ is collision free for $n \geq n_0$.
In this section we show that, asymptotically, our algorithm will sample such sequences almost surely.
Specifically, let $P_n$ denote the probability that our algorithm produces a set of samples~$X_n$ (one in each layer). 

\begin{lem}
\label{lem:pn}
$\lim_{n \rightarrow \infty} P_n = 1$.
\end{lem}

\begin{proof}
To prove Lemma~\ref{lem:pn}, we need to show that, for each ball~$B_{n,i}^{\Man}$,  with high probability we 
(i)~sample a milestone $\tau_{n,i}$ ``close'' to $FK(q_{n,i})$
and
(ii)~sample an IK solution of $\tau_{n,i}$ which lies in~$B_{n,i}^{\Man}$.
This is done, similar to~\cite[Thm.~34]{karaman2011sampling}, by using assumptions \emph{A1-A4} together with the appropriate choice of~$k$ (the number of IK solutions samples per waypoint), the radius of each ball~$r_n$ and the number of balls~$M_n$.
For simplicity, we assume that $\xi$ is not self intersecting.

Let $P_{n,i}^c$ be the event that the single ball, $B_{n,i}$,
does \emph{not} contain a vertex of the graph generated by our algorithm, when the algorithm samples $n$ waypoints and~$k$ IK solutions for each waypoint.
Thus,
$$
	P_{n,i}^c 
	\leq 
	\left(1 -  \eta_4(r_n/2	)\right)^n
	+
	\left(1 -  r_n/\ell_{\text{SM}}\right)^k.
$$
Here, the first component uses assumption \emph{A4} to bound the probability that no waypoint samples can be mapped to a ball at $q_{n,i}$ with radius $r_n / 2$.
The second component uses property \emph{P2} to bound the event that if such a waypoint was sampled, it's IK will not lie in~$B_{n,i}$.
Using the inequality $(1 - x)^y \leq e^{-xy}$,
\begin{equation}
	\forall i~
	P_{n,i}^c 
	\leq 
	e^{-n \eta_4(r_n/2)}
	+
	e^{-k_n r_n / \ell_{\text{SM}}}.
\end{equation}

Let $P_{n}^c$  denote the event that at least one ball does \emph{not} contain a vertex of the graph generated by our algorithm, when the algorithm samples $n$ waypoints and~$k$ IK solutions for each waypoint.
We can be bound~$P_{n}^c$ as follows:
$$
	P_{n}^c \leq \sum_{m=1}^{M_n}P_{n,m}^c = M_n P_{n,1}^c 
	\leq \frac{\ell}{\ell_n} 
	\left( 
		e^{-n \eta_4(r_n/2)}
		+
		e^{-k_n r_n / \ell_{\text{SM}}}
	\right).
$$

\noindent
Using the fact that $\ell_n = r_n$,
$$
P_{n}^c \leq \frac{\ell}{r_n} 
	\left( 
		e^{-n \eta_4(r_n/2)}
		+
		e^{-k_n r_n / \ell_{\text{SM}}}
	\right). 
$$

\noindent
Finally, $\sum_{n=1}^\infty P_{n}^c < \infty$ holds if
\begin{equation}
\label{eq:case1}
	\lim_{n \rightarrow \infty} \ln r_n + n \eta_4(r_n/2) = \infty,
\end{equation}
and
\begin{equation}
\label{eq:case2}
	\lim_{n \rightarrow \infty} \ln r_n + k_n r_n / \ell_{\text{SM}} = \infty,
\end{equation}
which can be easily satisfied using assumption \emph{A4} and by appropriately picking~$r_n$. 
By the Borel–Cantelli lemma~\cite{GS01}, $\sum_{n=1}^\infty P_{n}^c < \infty$ implies that
$$\lim_{n \rightarrow \infty} P_n^c = 0,$$
which, in turn,  implies that
$$\lim_{n \rightarrow \infty} P_n = 1,$$
which concludes the proof of Lemma~\ref{lem:pn}.
\end{proof}

\noindent
\textbf{Remark.} Assumption \emph{A4} ensures that there exists some function $\eta_4(\varepsilon) = \omega(\varepsilon^p)$   bounding from below the portion of $\Rpath$ that any ball lying on $\WellMan$ is mapped to.
The assumption that $\eta_4(\varepsilon) = \omega(\varepsilon^p)$
ensures that we can always choose $r_n$ that will satisfy that 
$\lim_{n \rightarrow \infty} r_n = 0$
\emph{and} Eq.~\ref{eq:case1}.
When $\eta_4(\varepsilon) = o(\varepsilon^p)$ this cannot be satisfied.

\subsection{Bounding the \frechet Distance and convergence to the optimal path}
Recall that $X_n = \{x_{n,1}, x_{n,2}, \ldots, x_{n,M_n} \}$ is a sequence of configurations such that $\forall i \ x_{n,i} \in B_{n,i}^\calM$.
Lemma~\ref{lem:pn} ensures that, asymptotically, our algorithm will sample such sequences as $n$ tends to $\infty$.
Furthermore, Lemma~\ref{lem:collision-free} ensures that the path $\xi_n$ connecting consecutive points in $X_n$ is collision free.

We are now ready to bound the \frechet distance between 
$\text{FK} ( \xi_n )$ and $\Rpath$ which will conclude the proof of Thm.~\ref{thm:ao}.
\begin{lem}
	$\lim_{n \rightarrow \infty} \frechetd(\text{FK}{(\xi_n)},\Rpath)  = 0$.
\end{lem}

\begin{proof}
Given two points, $x,x'$, let $\overline{x,x'}$ denote the straight-line segment connecting two points.
Note that for every points $x,x', y,y'$ in task space, it holds that
\begin{align}
\label{eq:frechet-bound1}
\frechetd(\overline{x,x'}, \overline{y,y'})
	& \leq
	\max\{
		\distT(x,y),
		\distT(x',y), \nonumber  \\
	& ~~~~~~~~~~	    
		\distT(x,y'),
		\distT(x',y')
		\}.
\end{align}

\noindent
Furthermore, 
\begin{align}
\label{eq:frechet-bound2}
	\frechetd(\text{FK}{(\xi_n)},\Rpath) 
	& = 
	\frechetd(\text{FK}{(\xi_n)}, \text{FK}{(\xi)}) \nonumber \\
	& \leq 
	\max_i \{ \frechetd(
		\overline{\text{FK}(x_{n, i}), \text{FK}(x_{n, i+1})}, \nonumber \\
	& ~~~~~~~~~~~~~~~
		\overline{\text{FK}(q_{n, i}), \text{FK}(q_{n, i1})}
			)\}. 
\end{align}
We have that 
(i)~$\forall i \distC(x_{n, i}, q_{n,i}) \leq r_n$
and using the triangle inequality and that $\ell_n = r_n$ that
(ii)~$\distC(x_{n, i}, q_{n,i+1}) \leq 2r_n$.
Plugging these into Eq.~\ref{eq:frechet-bound2} together with Eq.~\ref{eq:frechet-bound1} and using property \emph{P1},

\begin{equation}
\frechetd(\text{FK}{(\xi_n)},\Rpath)  \leq \eta_1(2r_n).
\end{equation}

\noindent
Finally, by the fact that 
$\lim_{n \rightarrow \infty} r_n = 0$ 
and by Property~\emph{P1} , we have that
$$\lim_{n \rightarrow 0} \eta_1(2r_n) = 0.$$
which will conclude the proof.
\end{proof}

%% file: appendix_figs.tex
\section{Supplemental Results}
\label{apndx:results}

\begin{figure*}[!t]
  \centering
  \minipage{0.32\textwidth}
  \includegraphics[width=\textwidth]{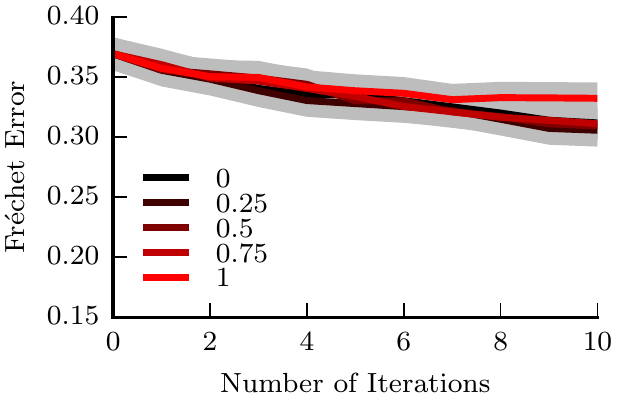}
  \endminipage 
  \minipage{0.32\textwidth}
  \includegraphics[width=\textwidth]{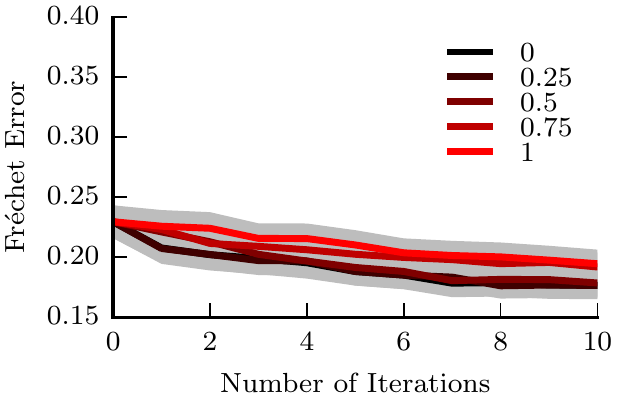}
  \endminipage
  \minipage{0.32\textwidth}
  \includegraphics[width=\textwidth]{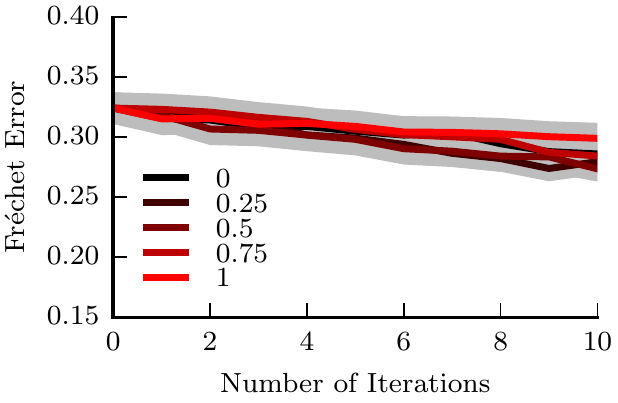}
   \endminipage
\caption{Across multiple paths and parameter settings, we compare the effect of various $p$-values for our hybrid densification strategy.}\label{fig:all_hybrid}\vspace{2mm}
\end{figure*} 

\begin{figure*}[!t]
  \centering
  \minipage{0.32\textwidth}
  \includegraphics[width=\textwidth]{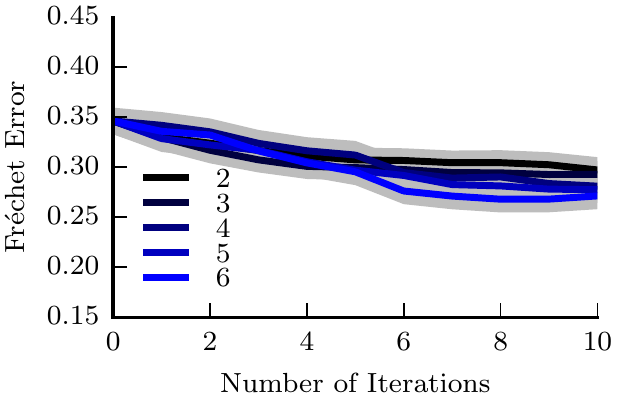}
  \endminipage 
  \minipage{0.32\textwidth}
  \includegraphics[width=\textwidth]{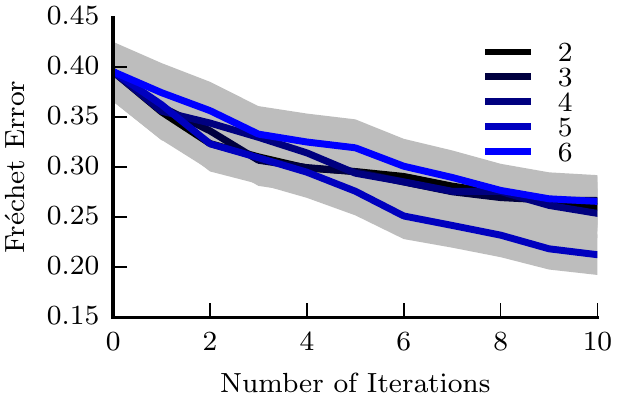}
  \endminipage
  \minipage{0.32\textwidth}
  \includegraphics[width=\textwidth]{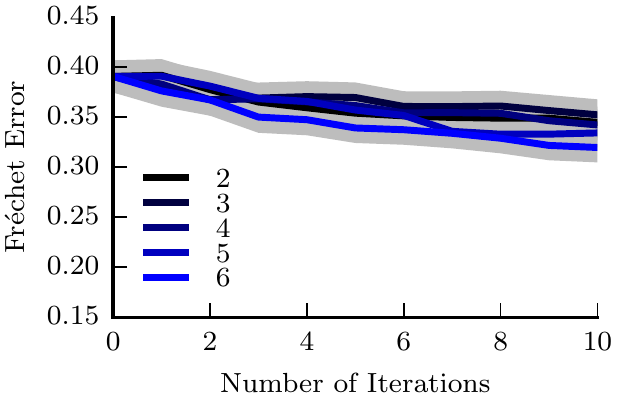}
  \endminipage
\caption{Across multiple paths and parameter settings, we compare the effect of various $m$-values for our local-then-global densification strategy.}\label{fig:all_segmented}\vspace{2mm}
\end{figure*} 

We present a more complete review of our experimental results that show comparisons across initial parameters and reference paths. 
Reviewing the experimental setup described in \sref{sec:experiment_densification}, we use the bimanual manipulator HERB to generate 100 instances of layered graphs for a given reference path $\bar{\xi}$, all with the same initial number of waypoints, IK solutions per waypoint, and level of subsampling resolution.  
For each problem we randomly place rectangular boxes in the vicinity of the robot. 
We then conduct many iterations of densification.
We first compare the parameters of each densification strategy before comparing the two strategies against each other. 
We then provide further comparisons of our algorithm with several state-of-the-art planners. 

\subsection{Densification Strategies}
Our two strategies, hybrid and local-then-global, each have one parameter.
For the hybrid strategy we compare $p$-values in the set $\{0, 0.25, 0.5, 0.75, 1\}$ across several paths and parameter settings in \figref{fig:all_hybrid}. 
In general, lower $p$-values (biasing local updates), produce paths with a shorter \frechet distance at each iteration. 
For the local-then-global strategy (referred to as, L-t-G) we compare $m$-values in the set $\{2, 3, 4, 5, 6\}$ across several paths and parameter settings in \figref{fig:all_segmented}. 
In general, mid-range $m$-values produce the best-quality results.

We select $p=0.25$ for the hybrid strategy and $m=5$ for the local-then-global strategy and compare their performance in \figref{fig:all_combo} and \figref{fig:all_methods}. The two methods achieve similar performance, with local-then-global edging out the hybrid strategy.

\begin{figure}[!t]
  \centering
  \minipage{0.45\columnwidth}
  \includegraphics[width=\textwidth]{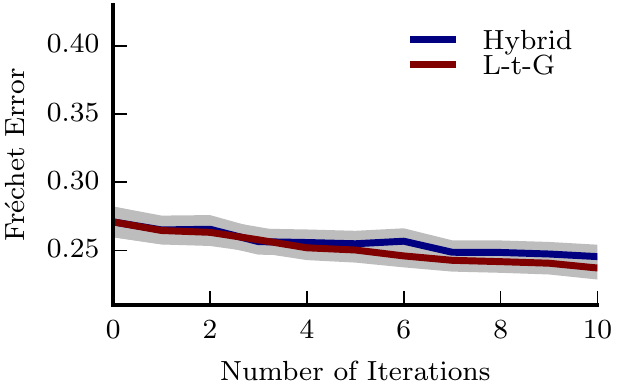}
  \endminipage 
  \minipage{0.45\columnwidth}
  \includegraphics[width=\textwidth]{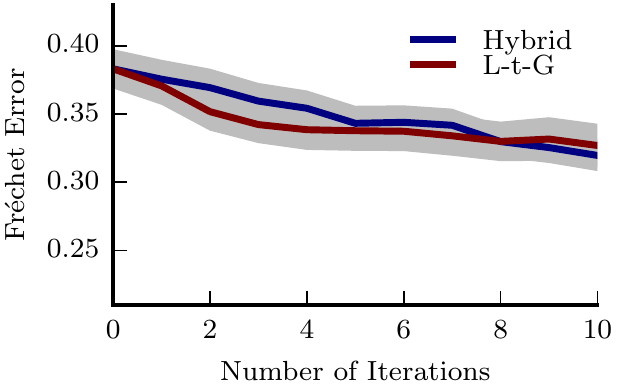}
  \endminipage
\caption{Evaluation of our two densification strategies, hybrid and local-then-global (L-t-G) across different paths and parameter settings.}\label{fig:all_combo}\vspace{2mm}
\end{figure} 

\begin{figure}[!t]
  \centering
  \minipage{0.45\columnwidth}
  \includegraphics[width=\textwidth]{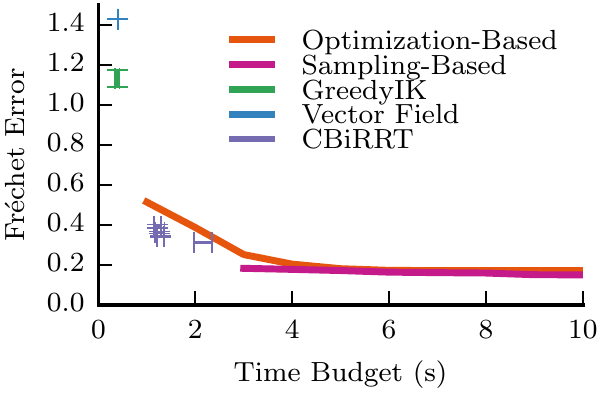}
  \endminipage 
  \minipage{0.45\columnwidth}
  \includegraphics[width=\textwidth]{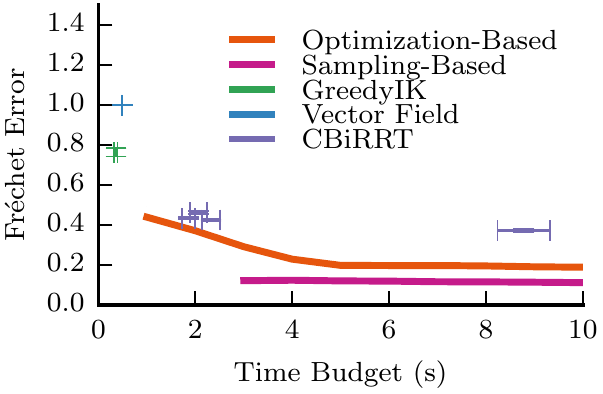}
  \endminipage
\caption{Empirical evaluation of our algorithm compared state-of-the-art planners on two additional random paths.}\label{fig:all_algo}\vspace{2mm}
\end{figure} 

\subsection{Cross-Algorithm Comparison}
In \sref{sec:results}, we compared our sampling-based algorithm with four other planners: an optimization-based approach~\cite{holladay2016distance}, a vector-field planner~\cite{srinivasa2016system}, a greedy inverse kinematic planner~\cite{srinivasa2016system} and CBiRRT (Constrained Bidirectional RRT)~\cite{berenson2009manipulation}. Further descriptions of each algorithm can be found in \sref{sec:results}. 
\figref{fig:all_algo} shows the performance of each planner on different paths, echoing the trends seen in \figref{fig:algo_compare}.
